\documentclass[conference]{IEEEtran}
\usepackage{times}

\usepackage{amsmath,amssymb,amsfonts}
\usepackage{algorithm}
\usepackage{algpseudocode}
\usepackage{graphicx}
\usepackage{textcomp}
\usepackage[table,x11names]{xcolor}
\usepackage[font=small]{caption}
\usepackage{float}
\usepackage{placeins}
\usepackage{hyperref}
\usepackage{multirow}
\usepackage{acronym}
\usepackage{amsthm}
\usepackage{xfrac}
\usepackage{latexcolors}
\usepackage{xspace}
\usepackage{soul}
\usepackage{wrapfig}
\usepackage{subcaption}
\usepackage{breqn}
\usepackage{tabularx}
\usepackage{booktabs}

\makeatletter
    \let\NAT@parse\undefined
\makeatother

\usepackage[hyperpageref]{backref}
\def\BibTeX{{\rm B\kern-.05em{\sc i\kern-.025em b}\kern-.08em
    T\kern-.1667em\lower.7ex\hbox{E}\kern-.125emX}}

\acrodef{RACER}{\textbf{R}isk-sensitive \textbf{A}ctor \textbf{C}ritic with \textbf{E}pistemic \textbf{R}obustness}
\acrodef{CVaR}{conditional value at risk}
\newcommand{\MethodName}[0]{\ac{RACER}\xspace}

\newcommand{\cvar}[0]{\textrm{CVaR}}
\newcommand{\var}[0]{\textrm{VaR}}
\newcommand{\R}[0]{\mathbb{R}}
\DeclareMathOperator*{\E}{\mathbb{E}}
\DeclareMathOperator*{\argmin}{arg\,min}

\newcommand{\tenth}[0]{\sfrac{1}{10}}

\newcommand{\KL}[0]{\textrm{KL}}

\newtheorem{thm}{Theorem}
\newtheorem{dfn}{Definition}
\newtheorem{lemma}{Lemma}

\colorlet{ourscolor}{skyblue!30}
\colorlet{ablatecolor}{gray!30}
\DeclareRobustCommand{\hlours}[1]{{\sethlcolor{ourscolor}\hl{#1}}}
\DeclareRobustCommand{\hlablate}[1]{{\sethlcolor{ablatecolor}\hl{#1}}}

\definecolor{fixcolor}{rgb}{1.0, 0.0, 0.0}
\definecolor{rebutcolor}{rgb}{0.0, 0.5, 0.0}
\definecolor{fixcolor}{rgb}{0.0, 0.0, 0.0}
\definecolor{rebutcolor}{rgb}{0.0, 0.0, 0.0}

\usepackage[square, numbers, sort&compress]{natbib}

\usepackage{multicol}

\begin{document}

\title{RACER: Epistemic Risk-Sensitive RL \\ Enables Fast Driving with Fewer Crashes}

\author{
Kyle Stachowicz, Sergey Levine \\
UC Berkeley \\
\texttt{kstachowicz@berkeley.edu, svlevine@eecs.berkeley.edu}
}

\maketitle

\begin{abstract}
Reinforcement learning provides an appealing framework for robotic control due to its ability to learn expressive policies purely through real-world interaction. However, this requires addressing real-world constraints and avoiding catastrophic failures during training, which might severely impede both learning progress and the performance of the final policy. In many robotics settings, this amounts to avoiding certain ``unsafe'' states. The high-speed off-road driving task represents a particularly challenging instantiation of this problem: a high-return policy should drive as aggressively and as quickly as possible, which often requires getting close to the edge of the set of ``safe'' states, and therefore places a particular burden on the method to avoid frequent failures.
To both learn highly performant policies and avoid excessive failures, we propose a reinforcement learning framework that combines risk-sensitive control with an adaptive action space curriculum.
Furthermore, we show that our risk-sensitive objective automatically avoids out-of-distribution states when equipped with an estimator for epistemic uncertainty.
We implement our algorithm on a small-scale rally car and show that it is capable of learning high-speed policies for a real-world off-road driving task. We show that our method greatly reduces the number of safety violations during the training process, and actually leads to higher-performance policies in both driving and non-driving simulation environments with similar challenges.
\end{abstract}

\IEEEpeerreviewmaketitle

\section{Introduction}
\label{sec:intro}

Reinforcement learning (RL) can in principle allow robots to perform complex and delicate behaviors, such as driving at high speed over rough terrain, while adapting to the particular environment in which they are trained. However, instantiating such methods while training directly in real-world environments
presents a unique set of challenges. The training process is no longer free of consequences, and catastrophic failures during training can impede learning progress, damage the robot, and require costly manual intervention where a person needs to reset the robot. This makes standard RL methods most difficult to apply in precisely the high-performance settings where they might be most beneficial.

In this paper, we study this challenge in the context of high-speed off-road driving. While learning to drive quickly over uneven terrain, high-speed crashes and rollover events can both damage the vehicle and disrupt the learning process, harming final performance. However, when the primary objective is to maximize driving speed, there is a tension between safety and performance: safety requires that the robot stays within a safe region, while achieving maximum performance requires the robot to operate at the edge of this set (as depicted in Fig.~\ref{fig:safe-boundary}).

\begin{figure}[t]
    \centering
    \includegraphics[width=0.9\columnwidth]{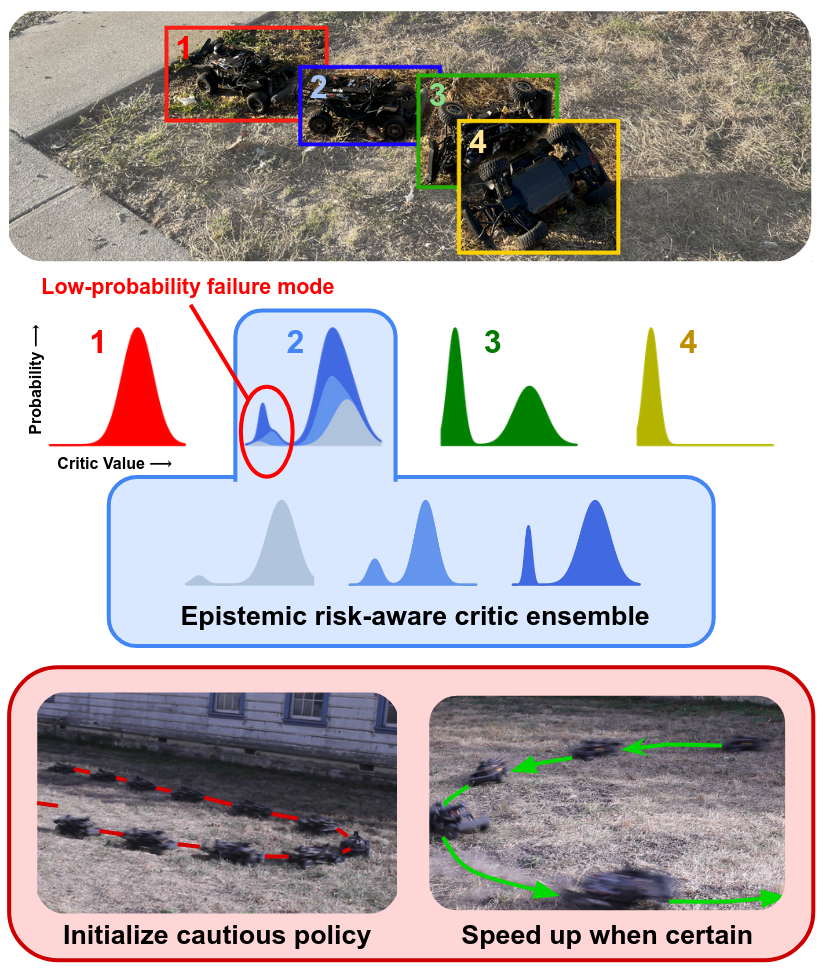}
    \caption{Our method enables high-speed driving with fewer crashes during training. Rare failure events (such as crashes or rollovers) often appear in the return distribution as a low-probability, low-return mode that do not contribute heavily to the expected value of the return.
    By applying a risk-sensitive actor objective (CVaR) to a distributional critic that incorporates epistemic uncertainty and can reason about these rare events, our method simultaneously modulates the robot's action limits and learns a risk-sensitive policy.}
    \label{fig:method_overview}
    \vspace{-18pt}
\end{figure}

Many existing approaches to safety in RL \cite{stooke2020responsive, so2023solving}
consider safety at the \textit{end} of training, once the policy has converged and the return is maximized. However, when training in the real world, it is also crucial to consider safety \textit{during} online training, before convergence \cite{bharadhwaj2020conservative}. During this period, the robot often encounters unfamiliar states during exploration as well as states that are not yet well fit by the training algorithm. The performance of the final learned policy depends critically on the robot's ability to avoid catastrophic failures consistently enough to learn high-performance policies.

We therefore aim to develop a reinforcement learning method that enables real-world robotic systems to learn high-performance behaviors (e.g., high-speed driving) while minimizing failures during training. Our key insight is twofold: firstly, to avoid rare failure events like rollovers or crashes, the agent must both model and effectively respond to low-probability outlier events, even when they are uncertain. Secondly, in many high-performance settings, it is relatively easy to obtain robust low-performance behavior (e.g., low-speed driving), and increase performance over time as the agent’s becomes more certain about the risky high-performance regime.

The primary contribution of this paper is \MethodName, a method for imbuing model-free RL agents with risk sensitivity to account for uncertainty over returns. We build upon distributional RL \cite{bellemare2017distributional}, which models the full distribution of returns rather than a mean point estimate. However, unlike the standard distributional RL setting, we explicitly model two types of uncertainty: \textit{aleatoric} uncertainty, which refers to the irreducible uncertainty in the returns (e.g. due to stochastic environment dynamics) and \textit{epistemic} uncertainty, corresponding to lack of knowledge of the true return distribution due to incomplete data and transient underfitting of newer data that has not yet been fit by the training process.

We model the return distribution as an ensemble of independently trained distributional neural networks. Each individual ensemble member models aleatoric uncertainty in its distributional output, and epistemic uncertainty is measured by the ensemble as a whole.
We propose to optimize the \ac{CVaR} of the distributional critics, which considers an expectation over the $\alpha$ worst-case distribution. We show that this naturally results in an agent that avoids taking actions leading to high \textit{epistemic} risk, as well as the avoidance of actions with highly stochastic returns (aleatoric uncertainty).

We also propose a risk-sensitive mechanism for scheduling exploration in this setting. \MethodName starts by using only a small (cautious) subset of the allowable action space, and then slowly increases the range of allowable actions over time according to a similar CVaR-based objective to that optimized by the actor. We increase the action limits only when the critics are confident that actions near the existing limits are safe.
This combination of risk-sensitive control and adaptive action space bounds provides for cautious exploration in unfamiliar situations, and avoids high-risk situations in familiar settings, leading to fewer crashes and better final performance.

We test \MethodName on a real-world tenth-scale autonomous vehicle performing aggressive off-road maneuvers, and show that our method allows the robot to reach $>$10\% higher speeds at convergence while cutting failures during the course of training by more than half, and almost entirely eliminating high-speed failures. We find that \MethodName compares favorably to several baselines in both driving and non-driving tasks, and demonstrate the importance of each component of our algorithm in reducing failures during training via ablation studies.

\section{Related Work}
\label{sec:background}
\noindent \textbf{Risk-sensitive RL.}
Several recent algorithmic developments aim to introduce risk-sensitive metrics to RL and control \cite{chapman2021risksensitive, chow2017riskconstrained}. \citet{tang2019worst} proposes a similar \ac{CVaR}-based objective for training agents that are robust to \textit{aleatoric} uncertainty about their environment. While these works demonstrate risk-sensitive training in simulation with on-policy methods, we propose a novel integration of CVaR with efficient off-policy algorithms in place of policy gradients \cite{prashanth2014policy} or model-based methods \cite{hakobyan2019riskaware}, enabling rapid training directly in the real world.

\citet{yang2021wcsac} propose a CVaR-based actor-critic algorithm to act as a safety critic, but their method is restricted to a Gaussian distribution representing aleatoric uncertainty only. \MethodName suppports a flexible critic distribution -- allowing representation of long-tailed failures that would be ignored with a Gaussian \cite{alexander2004comparison} -- and handles epistemic uncertainty directly (see Section~\ref{sec:epistemic}), allowing the resulting system to minimize safety violations \textit{during} training, rather than just at policy convergence.

\noindent \textbf{Constrained RL.}
The constrained RL approach casts safety as a \textit{constrained MDP} \cite{altman1999constrained}. Prior work \cite{achiam2017constrained, stooke2020responsive,tessler2018reward,wen2020safe,as2022constrained}
largely considers the problem of learning a policy to minimize test-time failures after the end of training. Safety critics train a discriminator model from offline datasets \cite{bharadhwaj2020conservative} or online (unsafe) interactions enabled via simulation \cite{srinivasan2020learning}. %

\noindent \textbf{Safe RL.}
Alternative approaches have focused on system-level safety \cite{garcia2015comprehensive} by merging classical techniques such as control barrier functions \cite{cheng2019end} or min-max robustness \cite{tamar2013scaling} with reinforcement learning. However, these control-theoretic approaches typically require a high-quality model of the system with simple parametric uncertainty \cite{choi2020reinforcement}, or require building an explicit model of how actions affect safety \cite{dalal2018safe}. In contrast, our method belongs to the class of model-free reinforcement learning algorithms and therefore makes no assumptions about the structure of the dynamics.

\begin{wrapfigure}[16]{r}{0.4\columnwidth}
    \vspace{-10pt}
    \centering
    \includegraphics[width=0.37\columnwidth]{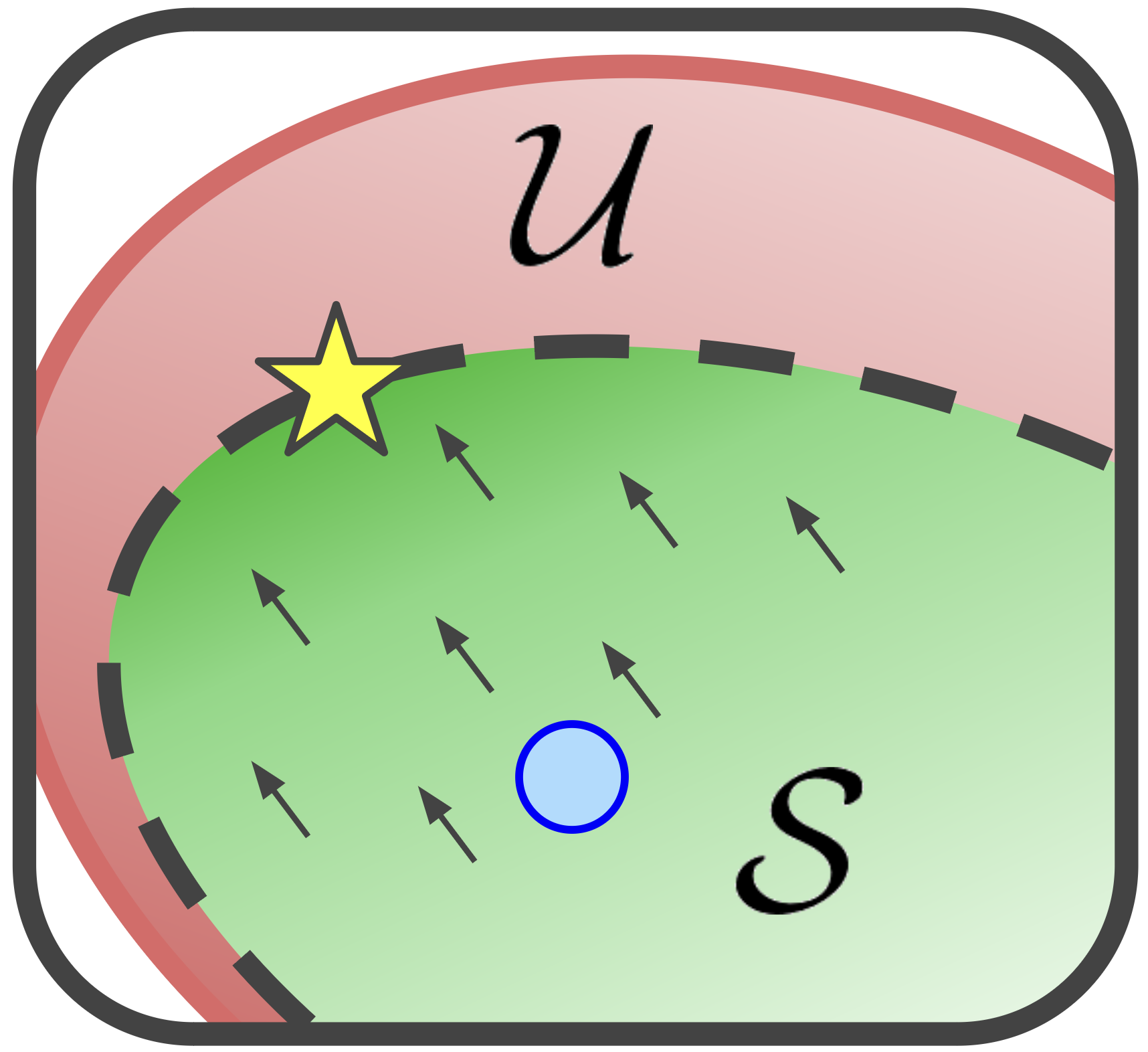}
    \caption{In high-speed driving, maximizing speed (arrows) requires operating on the boundary of the safety set $\mathcal{S}$ (optimal policy: star) to avoid unsafe states $\mathcal{U}$. Enforcing strict safety yields an overly conservative policy (blue).}
    \label{fig:safe-boundary}
\end{wrapfigure}
\noindent \textbf{High-Speed Driving.}
Substantial prior work considers high-speed driving in the classical control setting \cite{goel2020brakes}. These approaches require deriving a dynamics model from first principles and fitting parameters according to physically measured quantities or system identification. Unfortunately it is difficult to handcraft a dynamics model for off-road driving. Model-predictive control methods with learned models \cite{williams2018information} have been applied with some success in these domains. However, MPC-based methods have difficulty handling high-dimensional observations, as planning suffers from compounding error \cite{chua2018deep}. Because of these limitations, these works largely consider a simplified planar state-space. In the on-road driving setting, several works have considered the problem of modeling and preventing rollover, primarily in the quasistatic \cite{li2019rollover} case or by assuming a flat ground plane with constant wheel contact \cite{jalali2018model}, which cannot be assumed in the off-road setting. In comparison, our model-free approach aims to handle complex environments where an accurate hand-designed model might be difficult or impossible to obtain due to complex or chaotic system dynamics. In these settings a true safety \textit{guarantee} would out of necessity require a very conservative policy. Nevertheless, we find that our method greatly reduces failures during the training process and results in a highly performant policy.

Reinforcement learning has also been applied to high-speed driving in simulation \cite{Cai_2020, fabian2022drifting}. The closest prior work on a real robot is \cite{stachowicz2023fastrlap}, which also learns a fast driving policy through real-world interactions. However, unlike \citet{stachowicz2023fastrlap}, \MethodName is able to explicitly consider safety during the training process. We present a real-world comparison between the two methods in Sec.~\ref{sec:experiments}.

\section{Preliminaries}
\label{sec:preliminaries}

We first introduce standard notation and definitions in reinforcement learning and risk-sensitive optimization.

\subsection{Reinforcement learning}
Consider a Markov decision process $\mathcal M = (\mathcal S, \mathcal A, \mathcal T, \mathcal R)$ defined by a state space, an action space, a transition density, and a reward function respectively. We specifically consider the case of \textit{non-episodic} RL, in which the robot does not receive external state resets to some stationary initial distribution, except in the case of catastrophic failure, and these external resets are very expensive. This is reflective of the real world, where external resets may be unavailable or require human supervision.
We wish to learn a high-quality policy while minimizing the number of catastrophic failures (and hence, expensive external resets) across the course of training.

$Q$-learning methods approach the RL problem by learning a mapping $Q(s, a)$ from state-action pairs to the expected discounted return of the current policy. $Q(s, a)$ can be ``bootstrapped'' by iteratively performing Bellman updates:
\[Q^{k+1}(s, a) \gets \E_{r, s'}\left[r + \gamma \max_{a'} Q^k(s', a')\right]\]

Actor-critic methods extend $Q$-learning to continuous actions by sampling $a'$ from some function approximator trained to maximize $Q(s, a)$. These methods have proven relatively sample efficient \cite{schwarzer2023bigger}, motivating their use in learning directly in the real world rather than performing sim-to-real transfer \cite{zhao2020sim}. In such methods, we learn a policy $\pi$ and a critic function $Q^\pi(s, a)$ representing the expected discounted return from taking action $a$ for one step and then following the policy $\pi$. In deep RL, we typically parameterize both the actor and the critic with neural networks $\pi_\theta$ and $Q_\phi$.

\subsection{Distributional RL}
In the standard actor-critic setting we learn a critic that represents the \textit{expected} future discounted return and optimize an actor with respect to the critic. The \textit{distributional} RL perspective \cite{bellemare2017distributional, dabney2018distributional} instead learns the distribution of the returns $Z(s, a)$ and then maximizes its expectation $Q(s, a) = \E Z(s, a)$.

The simplest approach to distributional RL assumes the return distribution is Gaussian and parametrizes $Z(s, a)$ by its mean and a variance \cite{morimura2012parametric}. However, it is often desirable to represent $Z$ with a more flexible distribution. Two popular choices are a categorical distribution, in which the critic is optimized by minimizing KL divergence \cite{bellemare2017distributional}, and a quantile distribution with an L1-type loss \cite{dabney2018distributional}. These methods tend to enable better feature learning and are a key component of state-of-the-art RL algorithms in discrete domains \cite{hessel2018rainbow}.

In distributional RL, we can no longer minimize MSE to perform our approximate Bellman backup (as our critic is now a probability distribution rather than a point estimate). Instead, we typically minimize the KL-divergence to our target distribution $Z_{\textrm{targ}} = r + \gamma Z(s', a')$. When $Z$ is represented as a categorical distribution the KL-divergence may be undefined because the atoms of $Z$ and the target distribution are not aligned. The backup is approximated by projecting $Z_{\textrm{targ}}$ onto the atoms of $Z$; see \cite{bellemare2017distributional} for further details.

\subsection{Conditional Value at Risk}
When considering safety, is desirable to be conservative in the face of uncertainty: all else being equal, we should prefer the certain outcome over an uncertain one with the same expected value.
Particularly when there exists some binary indicator representing failure, returns will tend to follow a multimodal distribution: there is a high-probability, high-return mode corresponding to nominal behavior, and a low-probability, low-return mode corresponding to failure. Because failures require costly intervention, may damage the robot, and can harm the learning process, it is important to avoid this low-return mode, motivating a risk-sensitive formulation.

During training and exploration it is particularly important to be conservative with respect to \textit{epistemic} uncertainty; that is, risk corresponding to missing data or to a model that has not yet converged. Ideally our notion of risk should effectively handle both epistemic and aleatoric uncertainty (stochasticity). While it is possible to be sensitive to aleatoric risk by increasing the penalty for negative events \cite{whittle1981risk}, considering epistemic risk requires applying a risk metric to the actual (epistemic) return distribution.

\ac{CVaR} is one such instantiation of risk-sensitivity \cite{alexander2004comparison}. Let $Z$ be some scalar random variable. We can write \ac{CVaR} as:
\[\cvar_\alpha(Z) = \E[Z|Z<\var_\alpha(Z)],\]
where $\var(Z)$ (the \textit{value at risk}) is the $\alpha$-worst percentile value of $Z$. In our case we consider the \ac{CVaR} of the return distribution, which reflects the discounted return distribution of a state-action pair under the current policy.

\ac{CVaR} has many favorable properties as an optimization objective. It is less pessimistic than the worst-case value (the minimum value attainable, no matter how unlikely), which can be extremely noisy and even unbounded for stochastic systems. However, it tends to represent risk more effectively than value-at-risk \cite{alexander2004comparison}, while being more stable with respect to the underlying probability distribution \cite{chow2014algorithms}.

Additionally, CVaR is a \textit{convex} risk measure in the sense that $\cvar_\alpha(\mu Y + \nu Z) \le \mu \cvar_\alpha(Y) + \nu\cvar_\alpha(Z)$ for $\mu + \nu = 1$ and $\mu, \nu \ge 0$. We present a short proof in the appendix; see \citet{pflug2000some} for a detailed derivation.

We aim to optimize the CVaR of $Z$, which is difficult to compute in dynamic settings because of \textit{time inconsistency}~\cite{artzner1999coherent, boda2006time}. This property prevents optimization via dynamic programming, as the true optimal policy is non-Markov and depends on accrued rewards. Instead, we consider the return distribution and optimize $\cvar_\alpha(Z)$ under a Markov policy. While previous works take advantage of a closed-form expression for CVaR in the case of Gaussian return distribution~\cite{yang2021wcsac} or require use of on-policy (policy gradient-based), we use a more flexible categorical model of $Z$ and show that its CVaR can still be optimized using an efficient off-policy algorithm.

\section{Risk-Sensitive Exploration with RACER}
\label{sec:methods}
\MethodName (Fig.~\ref{fig:method-details}) is composed of three primary components: a risk-sensitive actor trained with a \ac{CVaR}-based objective, a distributional critic function incorporating both epistemic and aleatoric uncertainty, and an adaptive action limit that starts by restricting the policy to take ``cautious'' actions but becomes less restrictive over time. Both the actor and the adaptive action limits are updated based off of the distributional critic, and therefore the critic's epistemic and aleatoric uncertainty. The adaptive action limits act as a post-processing step on the actor network. We denote the distributional critic as $Z_\phi$, the (pre-limits) actor as $\bar\pi_\theta$, the adaptive limits as $L_\psi$, and the limited policy (combining the actor and the limits) as $\pi_{\theta}$.

\begin{figure*}
    \centering
    \begin{minipage}{0.5\textwidth}
        \centering
        \includegraphics[width=0.9\textwidth]{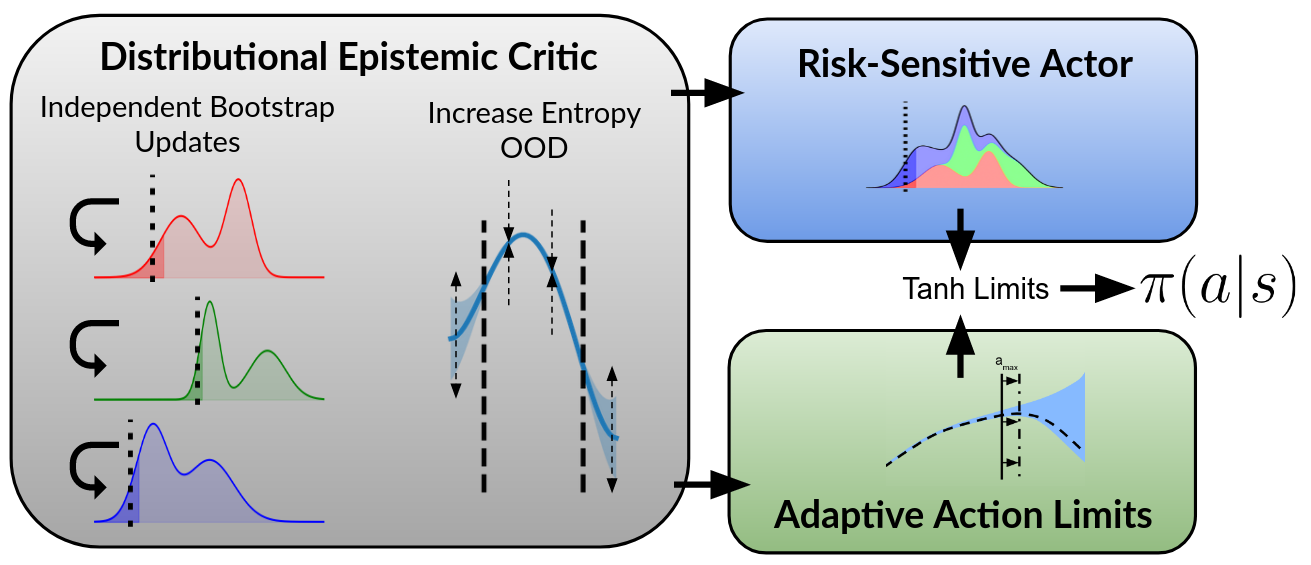}
        \caption{\MethodName and its three main components. A distributional critic captures epistemic uncertainty via ensembling and explicit entropy maximization beyond action limits. A risk-sensitive actor and adaptive action limits use the distributional critic to increase speed over time while reducing failures during training.}
        \label{fig:method-details}
    \end{minipage}
    \hfill
    \begin{minipage}{0.45\textwidth}
        \centering
        \includegraphics[width=0.75\textwidth, trim={2.5cm 0.5cm 2.5cm 1cm}]{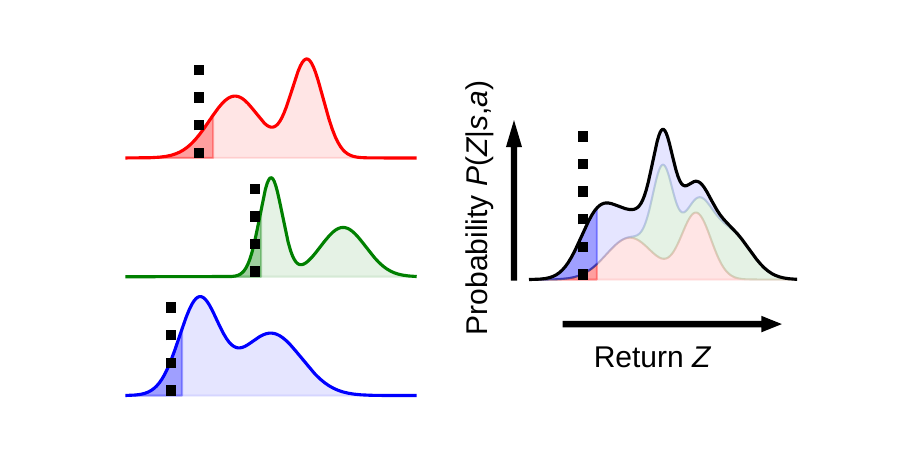}
        \caption{CVaR naturally accounts for epistemic uncertainty when applied to the mixture distribution output of an ensemble. When the ensemble disagrees about the distribution, the CVaR of their mixture prioritizes more pessimistic ensemble members.}
        \label{fig:cvar-ensemble}
    \end{minipage}
    \vspace{-10pt}
\end{figure*}

\subsection{Risk-Sensitive Deep RL with CVaR}
\label{sec:cvar-rl}

We optimize the CVaR objective via actor-critic RL, similarly to \citet{yang2021wcsac}.
To provide a flexible critic distribution, we learn a \textit{categorical} distributional value function $Z(s, a)$ \cite{bellemare2017distributional} with distributional Bellman backups, where each iteration is a sample-based approximation to the update
\[Z^{k+1} \leftarrow \argmin_Z \E\left[\KL(Z(s, a) || r(s, a) + \gamma Z^k(s', a'))\right],\]
where $a' \sim \pi^k(s')$.
Rather than optimizing the actor to maximize $Q(s, a) = \mathbb{E}Z(s, a)$, we optimize the $\cvar$ with respect to the actor parameters $\theta$ over $a \sim \pi_{\theta\psi}^k(s)$:
\[\pi^k \leftarrow \max_{\theta} \E_{a \sim \pi_{\theta\psi}(s)}\cvar_\alpha(Z^k(s, a)).\]
In practice, the updates above are implemented approximately, using one step of stochastic gradient descent on the objective with respect to $\theta$ at each iteration:
\[\mathcal{L}_\pi = -\E_{a \sim \pi_{\theta}}\cvar_\alpha(Z_\phi(s, a)),\]
following the standard procedure for implementing off-policy (distributional) deep actor-critic methods~\cite{lillicrap2015continuous,haarnoja2018soft,barth2018distributed}.

We represent $Z(s, a)$ as a neural network outputting a categorical distribution over discrete bins. Unlike prior work which represented $Z$ as Gaussian \cite{yang2021wcsac}, there is not a simple closed-form solution for the CVaR in this context. Instead, we provide a differentiable procedure for computing CVaR of a categorical distribution in Sec.~\ref{sec:impl-details}, which we then optimize to perform gradient descent on the actor. 

\subsection{Learning Distributional Epistemic Critics}
\label{sec:epistemic}

Because \ac{CVaR} is sensitive to the worst-case outcomes, its use as a policy objective implicitly penalizes states or actions resulting in highly uncertain distributions over returns. This means that \ac{CVaR} will implicitly try to avoid states and actions with a high degree of uncertainty.
There are two types of uncertainty: \textit{aleatoric} uncertainty, or irreducible stochasticity in the environment, and \textit{epistemic} uncertainty, which represents unknown quantities. While at convergence aleatoric uncertainty dominates, since the state distribution eventually stops changing as the policy converges, epistemic uncertainty is the more significant contributor to catastrophic failures during training, when the robot does not \emph{know} for sure whether or not a failure might occur. In light of this, we consider two possible approaches to incorporating epistemic uncertainty in distributional RL.

\subsubsection{Ensembled critics}
rather than a single predictor, an ensemble of independent neural networks can be applied to capture epistemic uncertainty \cite{ganaie2022ensemble, rahaman2021uncertainty, lakshminarayanan2017simple}. 
We train each ensemble member $Z_{\phi_i}$ individually with loss:
\[\mathcal{L}_{Z_{\phi_i}} = \E_{\substack{s, a, s' \sim \mathcal{D} \\ a' \sim \pi}}\Big[\KL\left(r+\gamma Z_{\phi_i'}(s', a')||Z_{\phi_i}(s, a)\right)\Big].\]

By maintaining a collection of models with varied parameters, which are only ever trained to agree on in-distribution data, we expect the ensembles to disagree on out-of-distribution data \cite{electronics10050567}.

\newcommand{\tail}[1]{\mathcal{T}_\alpha{#1}}
\newcommand{\emd}[2]{\left\lVert #1 - #2 \right\rVert_{\mathrm{EMD}}}
\newcommand{\tailemd}[2]{\emd{\tail{#1}}{\tail{#2}}}

\begin{thm}\label{thm:cvar-ensemble}
    Let $Z_i$ be real-valued random variables with density $p_i(z)$. Denote the random variable with density $\hat p(z) = \frac{1}{N}\sum_i p(z)$ as $\hat Z$. Then for $\alpha > 0$:
    \begin{equation}
        \cvar_\alpha(P) \le \frac{1}{N}\sum_i \cvar_\alpha(P_i)
    \end{equation}
    We call the positive difference $\frac{1}{N}\sum_i \cvar_\alpha(Z_i) - \cvar_\alpha Z_i$ the \textbf{ensemble CVaR gap}. (Proof in Appendix~\ref{appendix:cvar-proofs})
\end{thm}
\begin{thm}\label{thm:cvar-tail-emd}
    Assume $Z_i$ has finite first moment. Then we have:
    \begin{equation}
    \frac{1}{N}\sum_i\cvar_\alpha(Z_i) - \cvar_\alpha(\hat Z) \le \frac{1}{N}\sum_i \tailemd{\hat Z}{Z_i},
    \end{equation}
    where $\tail{X}$ is the \textbf{tail} of the distribution $X$ and $\emd{X}{Y}$ is the \textbf{earth mover's distance}. (Proof in Appendix~\ref{appendix:cvar-proofs})
\end{thm}
In practice we find that this bound is relatively tight (Appendix~\ref{appendix:cvar-gap-empirical}). This suggests that CVaR is conservative with respect to epistemic uncertainty represented as disagreement between individual members and the overall ensemble (Fig.~\ref{fig:cvar-ensemble}).

\subsubsection{Explicit entropy maximization}
we draw inspiration from conservative Q-learning \cite{kumar2020conservative}, in which out-of-distribution (OOD) actions are explicitly penalized by minimizing the critic on randomly-sampled actions. In our case, we wish to make OOD actions \textit{more uncertain} rather than decreasing their mean. We use entropy as a proxy for uncertainty, reducing the entropy of the return distribution $Z$ on seen datapoints and maximizing for sampled OOD actions. The modified critic loss for a single ensemble member using this form of epistemic uncertainty is then as follows with new terms in {\color{blue} blue}:
\begin{dmath}
    \mathcal L_Z = \E_{\substack{s, a, s' \sim \mathcal{D} \\ a' \sim \pi_{\theta}}} {\Big[}\KL\left(r + \gamma Z_{\phi'}(s', a')||Z_\phi(s, a)\right) \\
    {\color{blue}\quad\quad + \E_{a' \sim \pi_\theta}\mathcal{H}(Z_\phi(s, a')) - \E_{a' \sim \bar\pi_\theta}\mathcal{H}(Z_\phi(s, a'))}{\Big]},
\end{dmath}
where $\phi'$ is a delayed copy of $\phi$.

The out-of-distribution actions are drawn from any policy that samples actions outside of the action limits. We use $\bar\pi_{\theta}$, the \textit{pre-limit} policy; that is, the policy distribution induced by $\pi$ \textit{before} applying adaptive action limits.

\subsection{Adaptive Action Limits}
\label{sec:adaptive-limits}

In many real-world problems, including autonomous off-road driving, we have a strong prior in that ``cautious'' actions (for example, low speeds) tend to be safe. For example, for autonomous driving, our action space has two components: a servo command for the robot's steering in range $[-1, 1]$ and a velocity target for a low-level motor controller bounded by $[v^-, v^+]$, with the velocity target corresponding directly to cautious/risky actions. To adaptively adjust these bounds across the course of training, we want to increase our bounds whenever we are sufficiently certain about the return distribution $Z(s, a)$ for $a$ \textit{outside} (or at the edge) of those boundaries.

\begin{figure}[ht]
    \centering
    \includegraphics[width=0.8\columnwidth]{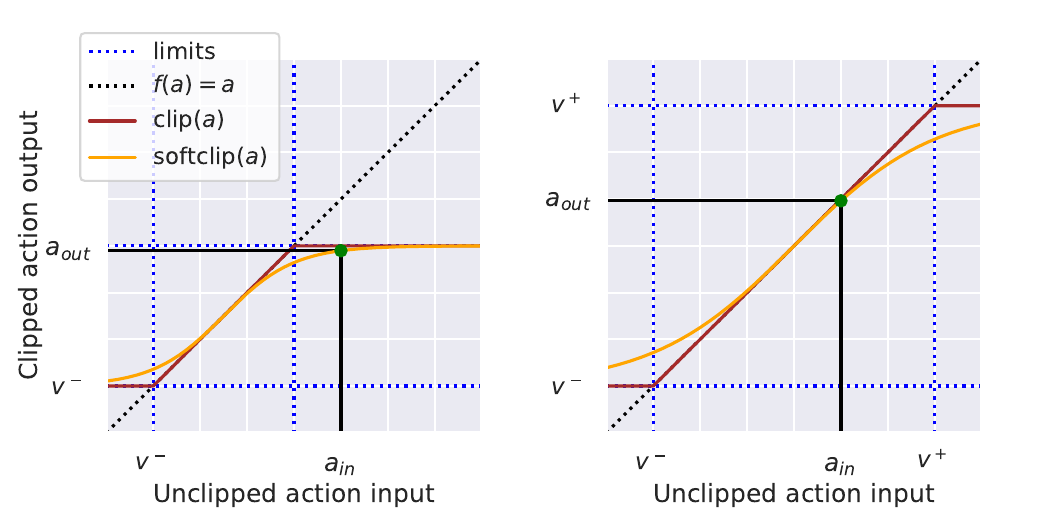}
    \caption{\MethodName applies a shifted $\tanh$ to the output of the actor to enforce a particular set of limits. As the limits are adjusted (left $\to$ right), the mapped values of actions already within the limits change very little, while outputs for outside values change significantly.}
    \label{fig:tanh_limits}
\end{figure}

Following \cite{stachowicz2023fastrlap}, we use a shifted $\tanh$ function to soft-clip the actor's output into the desired range:
\begin{equation}
\begin{gathered}
    \mathrm{softclip}(a, v^-, v^+) = \eta\tanh\left(\frac{1}{\eta}(a - \mu)\right) + \mu; \\ 
    \eta = \frac{v^+-v^-}{2}, \mu = \frac{v^++v^-}{2}
\end{gathered}
\end{equation}

This parameterization approximates $\mathrm{clip}(a, v^-, v^+)$, meaning that the bounds can be adjusted on the fly and for $a \in [v^-, v^+]$ the soft-clip will have minimal effect.

This differentiable approximation to $\mathrm{clip}$ allows us to adjust the action limits using the same risk-sensitive objective we use to optimize the policy. Because out-of-distribution actions (e.g., actions not appearing in any collected data) are subject to high epistemic uncertainty, maximizing the $\cvar$ objective will only grow the limits for as long as the critic is certain. We thus adaptively adjust the limit it over the course of training by taking gradient steps to maximize:
\[\mathcal{L}_L(v^+) = -\E_{s \sim \mathcal{D}, a \sim \pi_\theta}\left[\cvar_\alpha Z(s, \mathrm{softclip}(a, v^-, v^+))\right].\]

\subsection{Implementation Details}
\label{sec:impl-details}

For real-world training we apply sample-efficient RL by applying many more gradient steps on the critic than the number of steps we take in the environment \cite{chen2021randomized, d2022sample}.
As observed in the literature \cite{d2022sample, schwarzer2023bigger, chen2021randomized}, this type of sample-efficient RL algorithm often requires regularization to avoid overfitting to early samples from the non-stationary data distribution. Our method already uses ensembling \cite{chen2021randomized}; we additionally apply weight decay to the critic network. %

We compute the \ac{CVaR} of a categorical distribution in a differentiable fashion. As shown in Algorithm~\ref{alg:cvar}, we compute CVaR by finding the cumulative distribution of the distribution $W = P(Z|Z\le \var_\alpha(Z))$, which is equivalent to the CDF of $Z$ clipped to $\alpha$. The CDF is then differentiated to obtain $W$, and $\cvar_\alpha(Z)$ is computed as its expectation.

We provide a reference implementation of \MethodName in JAX. Additionally; we compile an extensive list of hyperparameters and architectural choices, as well as pseudocode and a reference implementation of \MethodName, on our website: \url{https://sites.google.com/view/racer-epistemic-rl}.

\begin{algorithm}
\caption{Compute CVaR}
\label{alg:cvar}
\begin{algorithmic}[1]
\State \textbf{Input:} $0 \le \alpha < 1$, distribution $Z$ via PDF $P_i$, atoms $Z_i$
\State \textbf{Intermediates:} CDF $C_i$, worst-cases PDF $\hat P_i$, CDF $\hat C_i$
\State \textbf{Output:} $\cvar_\alpha(Z)$
\Procedure{ComputeCVaR}{$P_i$, $Z_i$, $\alpha$}
    \State $C_i \gets \sum_{i'=1}^i(P_i)$  \Comment{Compute CDF (\texttt{np.cumsum})}
    \State $\hat C_i \gets \frac{\min(C_i, 1-\alpha)}{1-\alpha}$ \Comment{Worst-cases CDF\quad\quad\quad\quad\quad{\color{white}.}}
    \State $\hat P_i \gets C_i - C_{i-1}$ \Comment{Worst-cases PDF  (\texttt{np.diff}){\tiny\color{white}.}}
    \State \Return $\sum_i \hat P_i Z_i$
\EndProcedure
\end{algorithmic}
\end{algorithm}

\section{Experiments}
\label{sec:experiments}

Our experimental evaluation aims to study the performance of our proposed risk-sensitive RL algorithm, both on a real-world robotic platform and in a simulated environment that provides a more controlled setting for rigorous comparisons. The primary goal of our method is to train the fastest possible driving policy that can traverse an outdoor course with uneven terrain while avoiding catastrophic failure in the form of high-speed collisions or rollover events. We hypothesize that minimizing failures during training is instrumental to achieving high performance, though of course minimizing the number of failures is also inherently desirable to avoid damage to the hardware. We therefore report both the cumulative number of failures and the final speed for each experiment. We do \emph{not} aim to eliminate failures entirely, which is exceedingly difficult when learning from scratch.

\subsection{Real-World Experiments}

Our real-world evaluation of \MethodName uses a \tenth-scale remote-controlled car based on the F1TENTH platform \cite{f1tenth}. Following prior work \cite{stachowicz2023fastrlap}, we specify a course as a sparse sequence of checkpoints $\{c_i\}$, and define the reward function as ``speed-made-good'' $\vec v \cdot \vec g$, where $\vec v$ is the robot's velocity and $\vec g$ is the unit vector towards the next goal checkpoint. The robot's observation space consists of proprioceptive measurements (local velocity and IMU measurements) and a sequence of the next two goal vectors, generated using a GPS-based state estimator. Whenever one checkpoint $\{c_i\}$ is reached (measured via the robot's onboard GPS measurement), the goals are updated to reflect the next checkpoints.

We compare to FastRLAP \cite{stachowicz2023fastrlap}, a sample-efficient autonomous learning system for the high-speed driving setting, which uses an RL method based on SAC~\cite{haarnoja2018soft}.
As shown in Table~\ref{tab:results_real},
\MethodName learns high-speed policies while largely avoiding failures during the training process. 

\begin{table}[ht]
    \captionsetup{font=footnotesize}
\centering

\setlength{\tabcolsep}{3pt}
\begin{tabular}{c|c|c|c|c}
    Algorithm & \# Fails $\downarrow$ & \# Fails $>$2m/s $\downarrow$ & Lap time $\downarrow$ & Speed (m/s) $\uparrow$ \\
    \hline
    Vanilla SAC~\cite{haarnoja2018soft} & 4 & 2 & 5.9 & 2.92 \\
    \rowcolor{skyblue!30} \textbf{Ours} & \textbf{2} & \textbf{0} & \textbf{5.2} & \textbf{3.32} \\
\end{tabular}
\caption{Results from real-world experiments. Our algorithm accrues fewer failures over the course of training while reaching comparable or better final performance (measured by lap time). Because the policy begins with a cautious action limit, its early failures occur at lower speeds. By smoothly adapting action limits over time \MethodName transfers its low-speed knowledge to higher speeds, avoiding high-speed failures ($>$2m/s) that could cause damage to the robot during training.}
\label{tab:real}

    \label{tab:results_real}
\end{table}

\begin{figure}
    \centering
    \raisebox{0.2in}{\includegraphics[width=0.2\columnwidth]{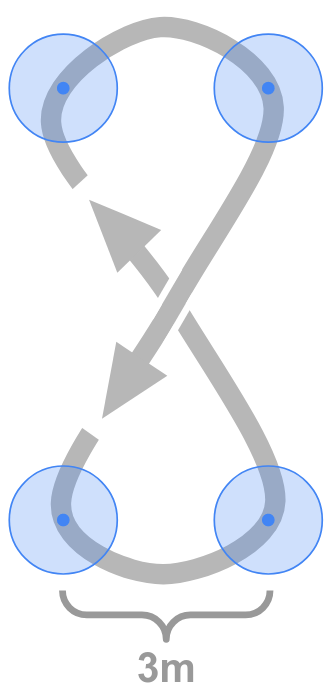}}
    \includegraphics[width=0.78\columnwidth]{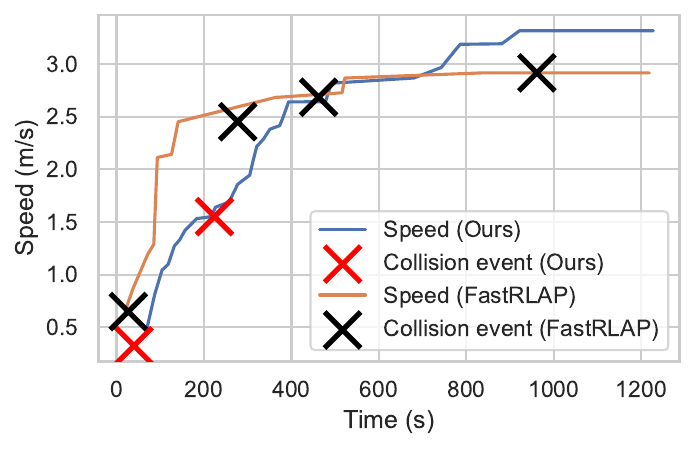}
    \caption{Real-world experimental results with a figure-eight driving course (left). Learning curves (right) show the robot's average speed-made-good over a single lap (equivalently, total lap length divided by lap time). \MethodName experiences fewer crashes during training, especially later in training when the robot is operating at high speeds.}
    \label{fig:real-world}
    \vspace{-15pt}
\end{figure}

Real-world learning curves are shown in Fig~.\ref{fig:real-world}. Please see the supplemental materials 
for videos of high-speed driving behavior learned by \MethodName.

\subsection{Simulated Comparisons and Ablation Studies}

Our simulated experiments are designed to evaluate the individual design choices in our method, as well as compare it to representative previously proposed approaches for both high-speed navigation and risk-aware/safe RL.

We conduct an comparison study in simulation to compare our work against prior methods and measure the effect of each of our design choices on two primary metrics: speed of the converged policy, measured by the robot's average speed at the end of training, and the cumulative number of safety violations incurred during training.

\begin{table*}[ht]
    \captionsetup{font=footnotesize}
\centering

\newcolumntype{r}{>{\color{rebutcolor}}c}
\begin{tabular}{l|c|c|c|c|r|r}
      & \multicolumn{2}{c|}{\texttt{Offroad-Flat}} & \multicolumn{2}{c|}{\texttt{Offroad-Bumpy}} & \multicolumn{2}{c}{\color{rebutcolor} \texttt{Cheetah}} \\
     Algorithm & \# Fails $\downarrow$ & Speed $\uparrow$ & \# Fails $\downarrow$ & Speed $\uparrow$ & \# Fails $\downarrow$ & Speed $\uparrow$ \\
     \hline
     SAC \cite{haarnoja2018soft} & $287\pm35$ & $5.87\pm0.31$ & $363\pm56$ & $4.23\pm0.23$ & $152\pm31$ & $7.67\pm1.70$ \\
     PPO \cite{schulman2017proximal} & $2750 \pm 812$ & $2.55 \pm 0.09$ & $16e3 \pm 4e3$ & $1.79 \pm 0.44$ & - & - \\
     SAC-Constrain \cite{achiam2017constrained} & $273\pm39$ & $5.93\pm0.44$ & $315\pm14$ & $4.24\pm0.21$ & $80\pm38$ & $7.45\pm1.05$ \\
     Safety critic \cite{srinivasan2020learning} & 245$\pm$34 & 5.91$\pm$0.73 & 265$\pm$86 & 3.99$\pm$0.34 & $63\pm68$ & $5.12\pm1.13$ \\
     WCSAC \cite{yang2021wcsac} & 253$\pm$71 & 2.33$\pm$3.03 & 445$\pm$59 & 3.85$\pm$1.35 & $14\pm16$ & $0.03\pm0.40$ \\
     MPPI$^*$ \cite{williams2018information} & $1000\pm275^*$ & $4.21\pm0.07^*$ & $2570\pm920^*$ & $3.60\pm0.08^*$ & 1$\pm$1 & 9.61$\pm$1.17 \\
     \rowcolor{ourscolor} \textbf{Ours} & \textbf{69$\pm$4} & \textbf{\textit{6.81$\pm$0.64}} & \textbf{100$\pm$9} & \textbf{\textit{4.59$\pm$0.30}} & \textbf{11$\pm$12} & \textbf{9.61$\pm$1.17} \\
     \rowcolor{gray!30} no epistemic & 122$\pm$17 & 6.1$\pm$0.41 & 165$\pm$29 & \textbf{\textit{4.60$\pm$0.36}} & 23$\pm$19 & \textbf{9.49$\pm$1.61} \\
     \rowcolor{gray!30} no adaptive limits & 108$\pm$9 & \textbf{6.92$\pm$0.46} & $152\pm35$ & \textbf{4.67$\pm$0.13} & 26$\pm$6 & \textbf{9.35$\pm$1.71} \\
\end{tabular}

\caption{Ablation study in simulation comparing \MethodName (\hlours{blue}) with ablations (\hlablate{gray}) as well as several prior works.
We consider the total number of safety violations across 250k steps of training (\# Fails) and the average speed of the final converged policy in m/s.
The risk-sensitive objective greatly reduces safety violations over the course of training. Adaptive action limits and epistemic uncertainty also significantly reduce the number of failures. MPPI results ($^*$) are extrapolated from convergence at 25k steps. Standard deviations are computed across 5 training runs; \textbf{bold} results are \textbf{significant} ($p<0.05$) in comparison to all prior work; \textbf{\textit{bold italic}} results are \textbf{\textit{weakly significant}} ($p<0.1$).}
\label{tab:sim}
\vspace{-10pt}

    \label{tab:results_sim}
\end{table*}

We consider two settings: an offroad driving setting and a classical locomotion task from the OpenAI gym. In the offroad driving setting, we use a simulated all-wheel-drive vehicle with suspension and Ackermann steering in two simulated environments: \verb|Offroad-Flat|, which is an infinite flat plane, and \verb|Offroad-Bumpy|, consisting of an infinite world procedurally generated using multi-scale Perlin noise \cite{perlin1985image}. Similar to the real-world experimental setup, the robot's objective is to drive towards the goal point as quickly as possible, with rewards defined by speed-made-good towards the goal, but we select a random goal from a normal distribution around the robot's current location each time a checkpoint is achieved.
We are interested in the cumulative number of failures across training, which we define by counting the cumulative number of rollover events (when the car is upside down, terminating the episode with zero reward). A rollover event can happen as a result of bumpy terrain or when the applied steering inputs are too aggressive at high speeds. Comparisons are averaged over 5 seeds each containing 250,000 timesteps.

The \verb|Cheetah| environment consists of OpenAI Gym's \verb|HalfCheetah-v4|, with the addition of a safety condition: the cheetah should remain upright, and failure to do so terminates an episode with zero return.    

In addition to standard \textbf{SAC}, we compare against several baselines:
\textbf{SAC-Constrain} \cite{achiam2017constrained, as2022constrained} penalizes failure via a Lagrange multiplier that is simultaneously adapted to keep a particular constraint violation rate. \textbf{WCSAC} implements distributional RL by parametrizing $Z$ as a normal distribution with learned mean and variance and optimizes a closed-form CVaR expression \cite{yang2021wcsac}. \textbf{Safety Critics} jointly learn a policy and a $Q$-function yielding the probability of remaining safe \cite{srinivasan2020learning, bharadhwaj2020conservative}, then perform rejection sampling on actions to find an action with $Q(s, a) < \epsilon$.
We include \textbf{PPO} \cite{schulman2017proximal} for completeness, though it is not well-suited for our setting due to sample inefficiency and does not achieve high performance. SAC-based baselines use the learned maximum-entropy policy directly for exploration, with the exception of safety critics which additionally performs rejection sampling on policy actions.

We also perform several ablations to consider the effects of several design choices: use of CVaR as a risk measure (Sec.~\ref{sec:cvar-rl}), epistemic uncertainty in the critic (Sec.~\ref{sec:epistemic}), and adaptive action-space limits (Sec.~\ref{sec:adaptive-limits}).

\begin{figure}[ht]
    \centering
    \includegraphics[width=0.9\columnwidth]{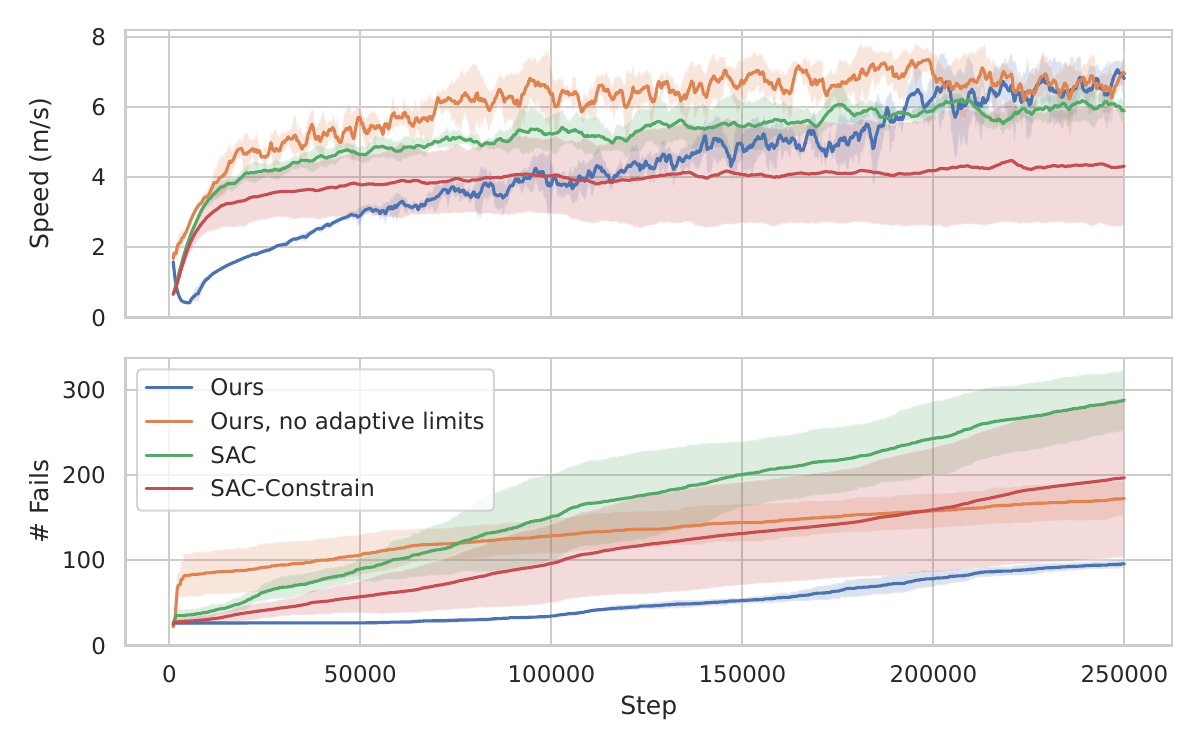}
    \caption{Average policy speed (top) and cumulative safety violations (bottom) for selected methods in the \texttt{Flat} environment. \MethodName learns a high-speed policy with few safety violations. Comparable methods (e.g. SAC-Constrain) reduces safety violations to an extent but result in an overly conservative policy.}
    \label{fig:learning_curves}
\end{figure}

We present our results in Tab.~\ref{tab:results_sim}. \MethodName exhibits far fewer total of violations during training than all baselines considered, while exhibiting slightly \textit{better} final performance. We hypothesize that when fewer failures occur during training, less of the model's representational capacity is wasted expressing high-loss outlier events, allowing for better estimation of the return earlier in the training process.

Perhaps surprisingly, non risk-sensitive methods tend to learn less performant final policies. They achieve relatively slow driving behavior in addition to a much higher failure rate. This is likely due to function approximation error in the critic $Q(s, a)$: MSE minimization tends to be very sensitive to outlier events (e.g. failures).

While safety critics do slightly reduce failures across training, the impact is relatively small in the driving setting, and negatively impacts performance in the locomotion setting. We note that \citet{srinivasan2020learning} suggests to train the policy on \textit{unfiltered} (unsafe) actions, which is not possible in the purely online setting where prior unsafe data is not available. We instead apply it to the online setting, which may limit its ability to accurately classify safe actions. The Gaussian critic model in WCSAC \cite{yang2021wcsac} exhibited unstable training and poor performance.

Ablating the adaptive control limiting mechanism presented in Sec.~\ref{sec:adaptive-limits} results in similar final policy performance to the full implementation \MethodName but incurs many more failures early in training Fig.~\ref{fig:learning_curves}. Removing the explicit handling of epistemic uncertainty described in Sec.~\ref{sec:epistemic} also causes increased failures and lower final speed.

\begin{figure}[ht]
    \centering
    \includegraphics[width=0.75\columnwidth,trim={0cm 0.7cm 0cm 1.3cm}]{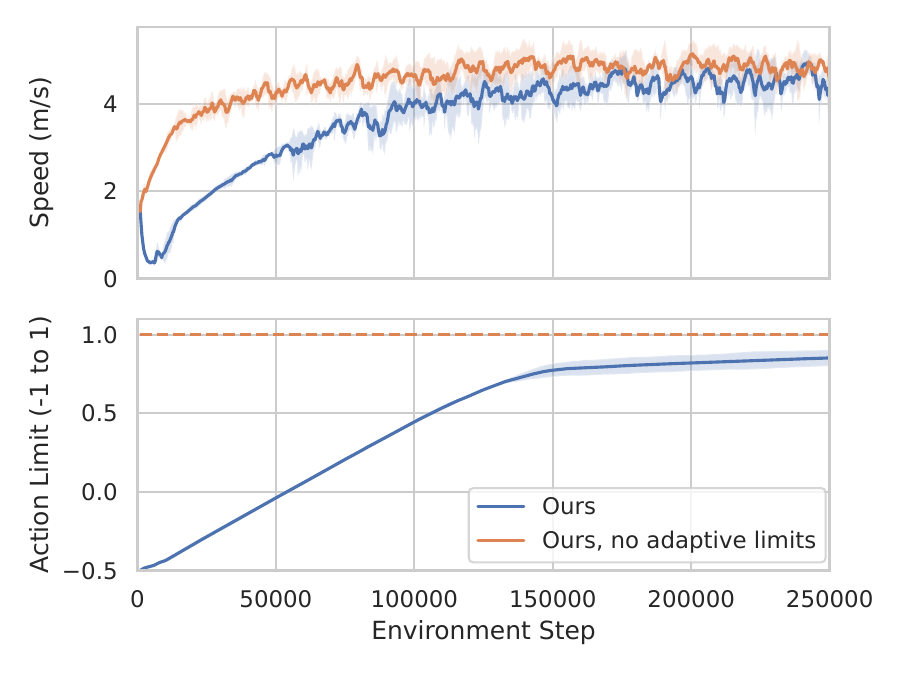}
    \caption{Comparison of \MethodName with and without adaptive limits. (Top) comparison of speed: standard CVaR formulation quickly learns a fast driving policy but fails frequently in the early stages of training. (Bottom) the CVaR+limits formulation slowly increases the maximum action over time, tapering off at the high end.}
    \label{fig:limits_comparison}
    \vspace{-10pt}
\end{figure}

\begin{figure}[ht]
    \centering
    \includegraphics[width=0.9\columnwidth]{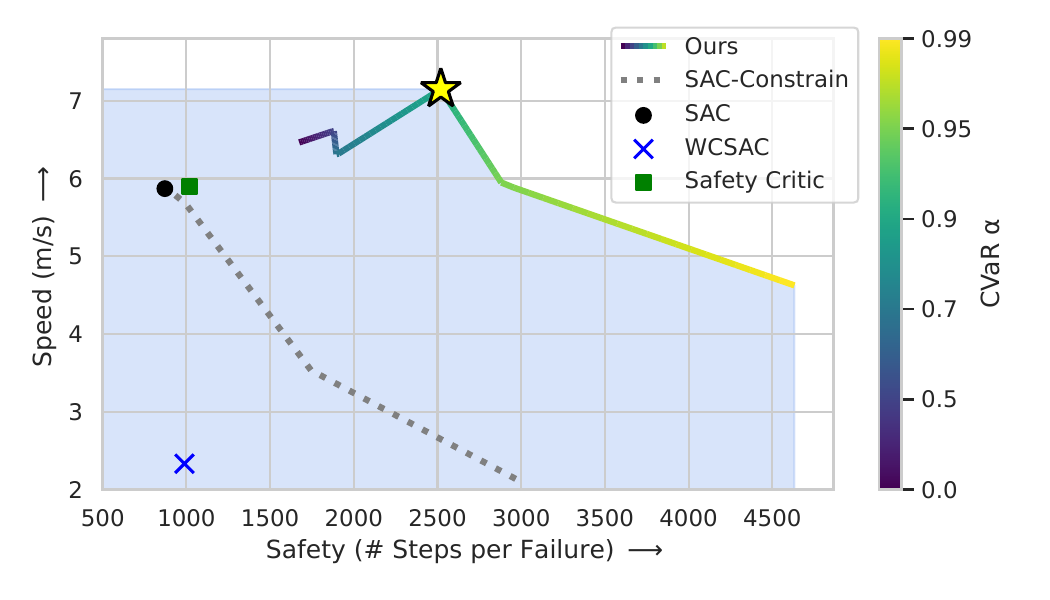}
    \caption{The parameter $\alpha$ trades off between safety and performance. While performance decreases dramatically at extreme risk-sensitivity ($\alpha \ge 0.95$), for moderate $\alpha \approx 0.9$ it actually results in \textit{increased} performance compared to risk-neutral settings ($\alpha=0$) while accruing far fewer safety violations. The blue region represents points that are dominated by \MethodName in both metrics for some value of $\alpha$.}
    \label{fig:pareto-frontier}
    \vspace{-15pt}
\end{figure}

The $\alpha$ parameter controls the conservatism of the CVaR metric. Fig.~\ref{fig:pareto-frontier} shows that as $\alpha \to 1$, the learned policy becomes more conservative and experiences fewer safety violations but somewhat decreased performance. However, \MethodName retains high performance for a wide range of $\alpha$.

\subsection{Analysis of the Learned Critic}

To better understand why risk-sensitive RL appears to be particularly helpful in the aggressive driving problem studied, we analyze the behavior of the distributional critic learned by \MethodName in risky states.

Fig.~\ref{fig:analysis-filmstrip-fail} shows the critic outputs in a case in which the policy experiences a failure. We highlight that only some members of the critic ensemble successfully model the low-probability mode corresponding to failure. Although the model knows that a failure is possible, the mean of the return distribution $Z$ actually remains relatively high. However, the CVaR objective only considers the tail of the distribution and thus decreases sharply in the face of ensemble disagreement, as in Theorem~\ref{thm:cvar-tail-emd}. This indicates that the CVaR-based actor should be more responsive to these uncertain events.

Fig.~\ref{fig:analysis-filmstrip-recovery} shows recovery behavior demonstrated by the learned policy. The critic still identifies the low-probability failure mode, but in this case is able to recover by steering in the ``downhill'' direction, a technique applied by human offroad driving experts \cite{Bruce2018Rollover}.

\begin{figure}[ht]
    \centering
    \includegraphics[width=0.9\columnwidth]{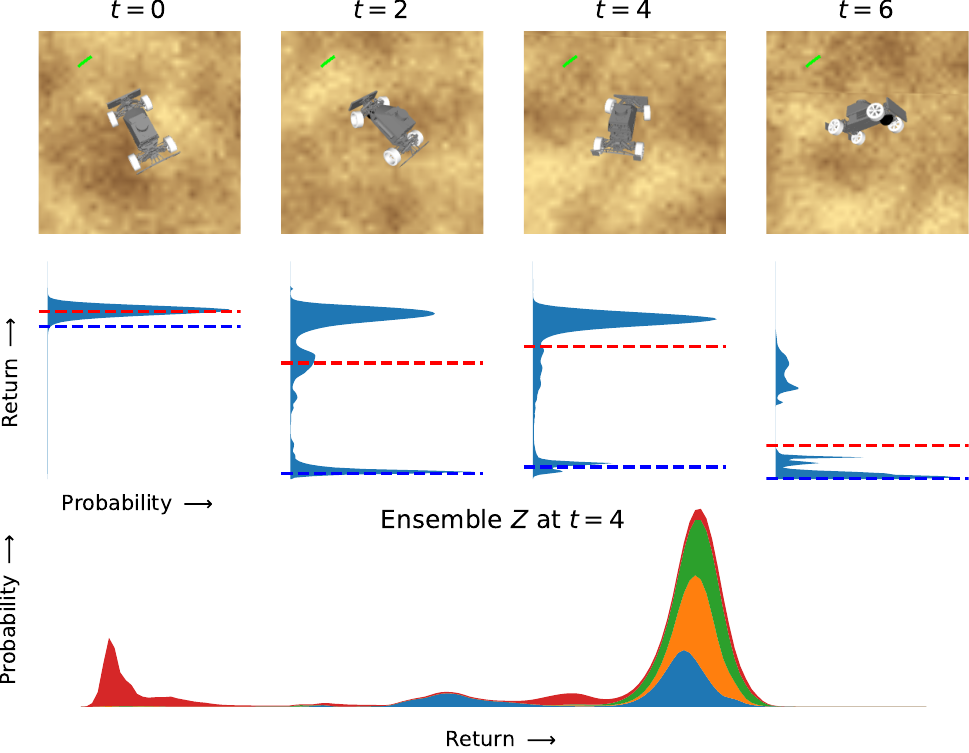}
    \caption{Sequence of simulated states ending in a failure (rollover). Below each image the distributional critic $Z(s, a)$ is displayed with its mean in {\color{red} red} and $\cvar_{\alpha=0.9}$ in {\color{blue} blue}. While the mean of the distribution remains high even when the critic has identified the low-probability failure mode (frame 4), the risk-sensitive CVaR metric appropriately penalizes the failure mode immediately as soon as it is detected. The final plot shows the individual ensemble member distributions at $t=4$. In this case only one ensemble member identifies the risky failure mode (red), highlighting the importance of handling epistemic uncertainty to avoid overconfidence.}
    \label{fig:analysis-filmstrip-fail}
\end{figure}

\begin{figure}[ht]
    \centering
    \includegraphics[width=0.9\columnwidth]{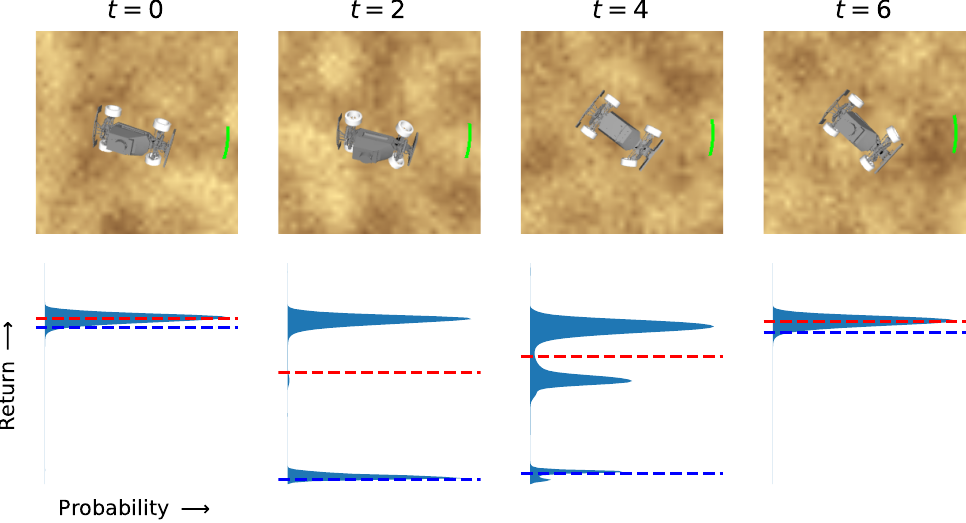}
    \caption{The risk-sensitive policy successfully recovers from a likely failure ($t=2$) by steering in the direction of the roll as soon as the critic detects a possible rollover. Goal direction indicated in green.}
    \label{fig:analysis-filmstrip-recovery}
    \vspace{-15pt}
\end{figure}

\section{Discussion, Limitations, and Future Work}

We presented \MethodName, a method for risk-sensitive real-world reinforcement learning applied to high-speed navigation that can acquire effective high-speed policies for driving on uneven terrain while minimizing the number of \emph{training time} failures. The core of our method is a novel extension of the CVaR principle to deep RL with epistemic uncertainty, which constitutes to our knowledge the first CVaR-based deep RL algorithm that makes use of full distributional critics, accounts for epistemic uncertainty during training (which is important for avoiding risky behavior in unfamiliar states), and works well for real-world robotic control tasks. Our experiments demonstrate that our approach leads to fewer training-time failures and actually enables both real-world and simulated robots to attain faster driving speeds.

Our method has a number of limitations. Firstly, we require the existence of some ``safe'' region of action space. While this assumption is reasonable in many cases -- often a ``do-nothing'' policy is quite easy to construct -- it is not applicable to all settings. While \MethodName does effectively reduce the number of failures, it does not eliminate them entirely in all cases. Indeed, our aim is not to provide a true ``safety constraint'' (as would be necessary, for example, for full-scale autonomous driving systems), but only to reduce the number of failures enough so as to obtain the best final policy performance. Eliminating failures entirely likely requires additional domain knowledge, and studying this direction further is an important topic for future work. However, our results do show that our approach reduces the number of failures by a large margin and, perhaps surprisingly, that reducing the number of failures during \emph{training} actually leads to more performant \emph{final} policies. We hope for this reason that our approach will provide an important stepping stone toward RL algorithms that are practical to use in the real world for safety-critical applications.

\FloatBarrier
\newpage

\section*{Acknowledgements}
This research was supported by DARPA ASNR, DARPA Assured Autonomy, and NSF IIS-2150826. The authors would like to thank Qiyang Li and Laura Smith for feedback on a draft version of this paper.

\nocite{*}
\bibliographystyle{plainnat}
\bibliography{references}

\FloatBarrier
\clearpage
\onecolumn
\begin{appendices}
\section{Proof of CVaR Theorems}\label{appendix:cvar-proofs}
\setcounter{thm}{0}
\setcounter{figure}{0}
\setcounter{table}{0}

We begin by restating the theorems listed in Section~\ref{sec:cvar-rl}. Let $Z_i$ be $N$ real-valued random variables with density $p_i(z)$ and cumulative distribution $P_i$. Additionally, let $\hat p(z) = \frac{1}{N}\sum_{i=1}^{N} p_i(z)$ be the density for the ensemble, with cumulative density function $\hat P(z) = \frac{1}{N}\sum_{i=1}^N P_i(z)$, and let $\hat Z$ be a random variable drawn from this distribution.

\begin{lemma}\label{appendix:cvar-convex}
    CVaR is convex, in the sense that $\cvar_\alpha[Z] \le \lambda \cvar_\alpha[X] + (1-\lambda)\cvar_\alpha[Y]$ when $Z$ is the mixture distribution $p_Z(\cdot) = \lambda p_X(\cdot) + (1-\lambda) p_Y(\cdot)$.
\end{lemma}
\begin{proof}
    Consider the expression:
    \[\inf_{a \in \R}\left[a + \frac{1}{1-\alpha}\E[\max(0, Z-a)]\right].\]
    Assuming smoothness, the infinimum occurs when the derivative of the expression is zero with respect to $a$:
    \[f(a, Z) = a + \frac{1}{1-\alpha}\E[\max(0, Z-a)]\]
    \[\frac{d}{da}f(a, Z) = 1 - \frac{1}{1-\alpha}\int_Z 1_{Z\ge a} dp(Z) = 1 - \frac{1}{1-\alpha}P(Z\ge a)\]
    Of course, this expression is equal to zero when $P(Z\ge a) = 1-\alpha$ (when $a$ is the VaR of $Z$). It then follows that:
    \[\inf_{a \in \R}\left[a + \frac{1}{1-\alpha}\E[\max(0, Z-a)]\right] = \frac{1}{1-\alpha}\E[\max(0, Z-\var_\alpha(Z))] = \cvar_\alpha(Z)\].
    With this alternative definition of CVaR (see \citet{pflug2000some} for a more rigorous derivation), we assume $a_1, a_2$ be the arg-min in this definition for $X$ and $Y$ respectively. Then, again following \citet{pflug2000some}:
    \begin{align*}
        \cvar_\alpha(Z) &= \cvar_\alpha(\lambda X + (1 - \lambda)Y) \\
        &\le \lambda a_1 + (1-\lambda) a_2 + \frac{1}{1-\alpha}\E\max(0, \lambda X + (1-\lambda)Y - \lambda a_1 - (1-\lambda) a_2) \\
        &\le \lambda \left[a_1 + \frac{1}{1-\alpha}\E\max(0, X - a_1)\right] + (1-\lambda) \left[a_2 + \frac{1}{1-\alpha}\E\max(0, Y - a_2)\right] \\
        &\le \lambda \cvar_\alpha(X) + (1-\lambda)\cvar_\alpha(Y)
    \end{align*}
\end{proof}

\begin{thm}\label{appendix:cvar-ensemble}
    Let $Z_i$ be real-valued random variables with density $p_i(z)$. Denote the random variable with density $\hat p(z) = \frac{1}{N}\sum_i p(z)$ as $\hat Z$. Then for $\alpha > 0$:
    \[\cvar_\alpha(P) \le \frac{1}{N}\sum_i \cvar_\alpha(P_i)\]
    We call the positive difference $\frac{1}{N}\sum_i \cvar_\alpha(Z_i) - \cvar_\alpha Z_i$ the \textit{CVaR gap}.
\end{thm}
\begin{proof}
    The result follows directly from applying Jensen's inequality to $\cvar_\alpha(P)$.
\end{proof}

\renewcommand{\tail}[1]{\mathcal{T}_\alpha{#1}}
\renewcommand{\emd}[2]{\left\lVert #1 - #2 \right\rVert_{\mathrm{EMD}}}
\renewcommand{\tailemd}[2]{\emd{\tail{#1}}{\tail{#2}}}

\begin{dfn}
    Let $Z$ be a real-valued random variable. Then for $0 < \alpha < 1$ define the \textbf{tail distribution} $\tail{Z}$ as the distribution with probability mass $\frac{1}{1-\alpha}p(z)$ over the support:
\[\mathrm{supp}(\tail{Z}) = (-\infty, \var_\alpha(Z)] \, \cap \, \mathrm{supp}(Z).\]
We also denote the density function of $\tail{Z}$ as $\tail{p(z)}$ and its cumulative density as $\tail{P(z)}$.
\end{dfn}
Note that $\cvar_\alpha(Z) = \E[\tail{Z}]$. We now provide the proof of our first theorem:
\begin{dfn}
Let $X, Y$ be real-valued random variables with cumulative distribution functions $\Phi_X, \Phi_Y$ respectively. Define the \textbf{earth-mover's distance} $\emd{X}{Y}$ as:
\[\int_\R \lvert \Phi_X(x) - \Phi_Y(x) \rvert dx\]
\end{dfn}
\begin{thm}\label{appendix:cvar-tail-emd}
    Let $Z_i$ be random variables with density $p_i(z)$. Assume $Z_i$ has finite first moment and denote the mixture distribution as $\hat Z$ with density $\hat p(z)$. Then we have:
    \[\frac{1}{N}\sum_i\cvar_\alpha(Z_i) - \cvar_\alpha(\hat Z) \le \frac{1}{N}\sum_i \tailemd{\hat Z}{Z_i}\]
\end{thm}
\begin{proof}
    \begin{align*}
        \frac{1}{N}\sum_{i=1}^N\emd{\tail{Z_i}}{\tail{\hat Z}} &= \frac{1}{N}\sum_{i=1}^N\int_\R \lvert \tail{\hat P(z)} - \tail{P_i(z)} \rvert dz \\
        &\ge \frac{1}{N}\sum_{i=1}^N\int_{\R} \left(\tail{\hat P(z)} - \tail{P_i(z)}\right) dz
    \end{align*}

    Integration by parts gives that $\int_{\R} (P(x) - Q(x)) \,dx = \left[x (P(x) - Q(x))\right]_{-\infty}^{\infty} + \int_a^b (q(x) - p(x)) x \, dx$:
    \begin{align*}
        \frac{1}{N}\sum_{i=1}^N\left[\lim_{z\to\infty}z\left(\tail{P_i(z)} - \tail{\hat P(z)} + \tail{P_i(-z)} - \tail{\hat P(-z)}\right) + \int_\R \left(\tail{p_i(z)} - \tail{\hat p(z)}\right)z\,dz \right]
    \end{align*}
    The first term vanishes because all cumulative density functions approach 0 and 1 at $\pm \infty$ respectively, and the distributions have finite moments. Then we are left with:
    \[\frac{1}{N}\sum_{i=1}^N \tailemd{Z_i}{\hat Z} \ge \frac{1}{N}\sum_i \E\tail{Z_i} - \E\tail{\hat Z} = \frac{1}{N}\sum_i \cvar_\alpha{Z_i} - \cvar_\alpha \hat Z\]
\end{proof}

\section{Emperical Analysis of CVaR and Ensemble Divergence}\label{appendix:cvar-gap-empirical}
While Theorem~\ref{appendix:cvar-tail-emd} provides an upper-bound on CVaR gap based on the tail EMD, it does not provide a lower bound stricter than that in Theorem~\ref{appendix:cvar-ensemble}. In fact, for any tail-EMD $D \ge 0$ it is possible to construct an adversarial ensemble distribution with zero CVaR gap; for example (with $\alpha = 0.5$):
\[p_1(z) = \begin{cases}
    0.5 & z = D \\
    0.5 & z = D/2 \\
    0 & \mathrm{otherwise}
\end{cases}\quad\quad p_2(z) = \begin{cases}
    0.5 & z = D \\
    0.5 & z = -D/2 \\
    0 & \mathrm{otherwise}
\end{cases}\quad\quad \hat p(z) = \begin{cases}
    0.5 & z = D \\
    0.25 & z = D/2 \\
    0.25 & z = -D/2 \\
    0 & \mathrm{otherwise}
\end{cases}\]

Nonetheless, in practice we find that tail-EMD empirically correlates extremely well with the CVaR gap for typical distributions (bounded, smooth). We probe this relationship in the setting of random gaussian-mixture ensemble distributions using the following procedure:
\begin{enumerate}
    \item Select a random ``base'' distribution as a Gaussian mixture with $K=3$ components with parameters sampled randomly.
    \item Create $N=3$ ensemble members by randomly sampling parameters around the ``base'' parameters.
    \item Find the mixture distribution of the ensemble members then compute $\frac 1 N \sum_i \tailemd{Z_i}{\hat Z}$ and the CVaR gap.
\end{enumerate}

Plotting the CVaR gap and tail-EMD yields Figure~\ref{appendix:fig:cvar-tail-emd}. There is a clear relationship between the two. This indicates that pessimism in RACER correlates extremely well with ensemble divergence.

\begin{figure}
    \centering
    \includegraphics[width=0.3\textwidth]{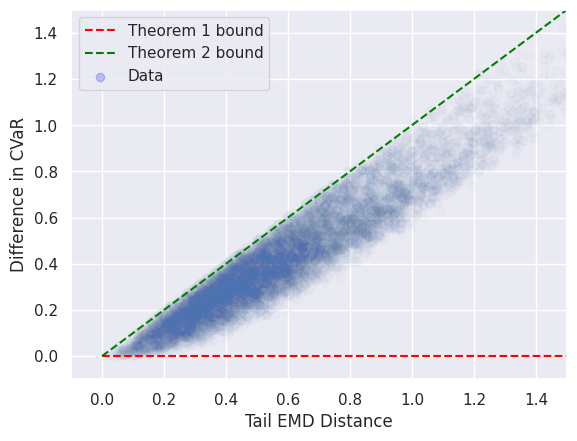}
    \caption{Tail EMD and CVaR gap for randomly sampled mixtures of Gaussians. CVaR gap correlates very well with tail EMD, indicating that the bound provided in Theorem~\ref{appendix:cvar-tail-emd} is relatively tight.}
    \label{appendix:fig:cvar-tail-emd}
\end{figure}

\end{appendices}

\end{document}